\documentclass[anon]{l4dc2026}

\usepackage{amsmath,amssymb,amsfonts}
\usepackage{graphicx}
\usepackage{booktabs}
\usepackage{algorithm}
\usepackage{algorithmic}
\usepackage{url}
\usepackage{hyperref}
\usepackage{enumitem}
\usepackage{tikz}
\usepackage{mathtools}
\usepackage{multirow}
\usetikzlibrary{shapes.geometric, arrows.meta, positioning, calc}

\title[Algorithm-Relative Trajectory Valuation]{Algorithm-Relative Trajectory Valuation in Policy Gradient Control}

\author{
  \Name{Shihao Li} \Email{shihaoli01301@utexas.edu} \AND
  \Name{Jiachen Li} \Email{jiachenli@utexas.edu} \AND
  \Name{Jiamin Xu} \Email{jxu@utexas.edu} \AND
  \Name{Christopher Martin} \Email{cbmartin129@utexas.edu} \AND
  \Name{Wei Li} \Email{weiwli@utexas.edu} \AND
  \Name{Dongmei Chen} \Email{dmchen@utexas.edu}\\
  \addr The University of Texas at Austin
}

\jmlrvolume{vol}
\jmlryear{2026}
\jmlrworkshop{Learning for Dynamics \& Control Conference (L4DC) 2026}

\begin{document}

\maketitle

\begin{abstract}
We study how trajectory value depends on the learning algorithm in policy-gradient control. Using Trajectory Shapley in an uncertain LQR, we find a robust \emph{negative} correlation between a trajectory's information content---Persistence of Excitation (PE)---and its marginal value under vanilla REINFORCE (e.g., $r\!\approx\!-0.38$). We prove a \emph{variance-mediated mechanism}: (i) for fixed energy, higher PE yields lower gradient variance; (ii) near saddle regions, higher variance increases the probability of escaping poor basins and thus raises marginal contribution. When the update is stabilized (state whitening or Fisher preconditioning), this variance channel is \emph{neutralized} and information content dominates, flipping the correlation \emph{positive} (e.g., $r\!\approx\!+0.29$). Hence, trajectory value is \emph{algorithm-relative}: it emerges from the interaction between data statistics and update dynamics. Experiments on LQR validate the two-step mechanism and the flip, and show that decision-aligned scores (Leave-One-Out) complement Shapley for pruning near the full set, while Shapley remains effective for identifying high-impact (and toxic) subsets.
\end{abstract}




\section{Introduction}
\label{sec:introduction}

Data-driven control via reinforcement learning (RL) has enabled impressive results in robotics and autonomy \cite{kober2013reinforcement, barto2021reinforcement}, yet the opacity of learned controllers complicates verification and trust in safety-critical applications \cite{molnar2025, wang2022artificial}. This motivates \emph{data attribution}: understanding which training experiences most influence a learned policy. While classical system identification focuses on parameter estimation accuracy, data-driven control requires understanding which data most improves the \emph{final controller performance}---a question of marginal contribution rather than information content alone. A natural intuition, rooted in system identification \cite{ljung1998system}, is that trajectory value scales with \emph{information content}, quantified by Persistence of Excitation (PE). Modern ML attribution tools---Influence Functions \cite{koh2017understanding, NEURIPS2020_e6385d39} and Shapley values \cite{shapley1953value, ghorbani2019data}---offer principled frameworks for quantifying these marginal performance contributions. However, policy-gradient (PG) learning introduces algorithm-data interactions absent in i.i.d.\ settings: temporal correlation, optimizer-dependent updates, and gradient variance \cite{wang2024capturing, 10.1145/3527448}.

\paragraph{Key finding and explanation.}
Using Trajectory Shapley to value data for a PG controller on uncertain LQR \cite{ozaslan2022computing, zhang2023revisiting, sun2024linear}, we discover a \emph{robust negative correlation} between PE and trajectory value under vanilla REINFORCE (Spearman $r \approx -0.38$). This contradicts the natural information-content intuition: high-PE trajectories, despite carrying more system knowledge, are valued \emph{less}. We prove this is an \emph{algorithmic artifact} arising from a variance-mediated mechanism. For fixed trajectory energy, higher PE yields \emph{lower} gradient variance (Theorem~\ref{thm:PE_to_Var}). Near saddle points in parameter space, higher variance \emph{increases} escape probability, raising marginal contribution (Theorem~\ref{thm:Var_to_value}). Thus vanilla REINFORCE favors high-variance updates from low-PE trajectories for exploration, while underutilizing stable, informative gradients from high-PE data.

\paragraph{Proof via algorithmic intervention.}
Stabilizing the optimizer (state whitening or natural gradient) \emph{neutralizes} the variance channel, allowing information content to dominate. The PE-value correlation \emph{flips positive} ($r \approx +0.29$; Theorem~\ref{thm:stabilized}), confirming our thesis: \textbf{trajectory value is algorithm-relative}---it emerges from data-algorithm interaction, not data alone.

\paragraph{Contributions.}
We make four main contributions. First, we demonstrate a robust negative correlation between PE and trajectory value under vanilla policy gradient, contradicting the natural intuition from system identification that more informative data should be more valuable. Second, we develop a theoretical framework establishing the variance-mediated mechanism: high PE leads to low variance, which in turn reduces marginal value (Theorems~\ref{thm:PE_to_Var}--\ref{thm:Var_to_value}). Third, we provide causal evidence via algorithmic stabilization (state whitening and natural gradient) showing that the negative correlation is an algorithmic artifact (Theorem~\ref{thm:stabilized}). Fourth, we derive practical insights for data curation: decision-aligned Leave-One-Out valuation outperforms classical Shapley for pruning decisions, while Shapley excels at identifying high-impact and toxic trajectory subsets.

\paragraph{Related work.}
Data valuation in supervised learning \cite{ghorbani2019data, NIPS2017_8a20a862, kwon2022betashapleyunifiednoisereduced} and RL influence methods \cite{hu2025snapshotinfluencelocaldata, deng2025survey} have not studied how PG algorithm design affects trajectory value. Work on trajectory importance \cite{schaul2015prioritized, munos2016safe} focuses on sampling efficiency rather than attribution. We bridge control theory (PE, LQR) and ML (Shapley, gradient dynamics) to reveal algorithm-relativity in data valuation.

\paragraph{Organization.}
Section~\ref{sec:preliminaries} formalizes the LQR setting and Shapley valuation framework. Section~\ref{sec:theory} develops our theoretical framework in three steps: proving high PE yields low variance, showing high variance increases marginal value near saddles, and demonstrating that stabilization neutralizes this mechanism. Section~\ref{sec:experiments} validates these predictions experimentally and derives practical curation insights. Section~\ref{sec:discussion} discusses implications and limitations.


\section{Problem Formulation and Background}
\label{sec:preliminaries}

We study trajectory valuation for policy-gradient learning on the uncertain LQR problem. This section establishes the setting, introduces Trajectory Shapley as our valuation tool, and defines key quantities (PE, gradient variance).

\subsection{Setting: Uncertain LQR with Policy Gradient}

\paragraph{System and objective.}
Consider discrete-time linear system
\begin{equation}
x_{k+1} = Ax_k + Bu_k + w_k, \quad w_k \sim \mathcal{N}(0, \Sigma_w),
\label{eq:system}
\end{equation}
with unknown $(A,B)$ and stage cost $\ell_k := x_k^\top Q x_k + u_k^\top R u_k$ ($Q \succeq 0$, $R \succ 0$). We minimize $J(K) := \mathbb{E}[\sum_{k=0}^{H-1} \ell_k]$.

\paragraph{Stochastic policy and REINFORCE.}
The agent uses linear-Gaussian policy $u_k = -Kx_k + \varepsilon_k$ with $\varepsilon_k \sim \mathcal{N}(0, \sigma_a^2 I_m)$. For trajectory $\tau = \{(x_k, u_k, r_k)\}_{k=0}^{H-1}$ with $r_k := -\ell_k$, the REINFORCE estimator is
\begin{equation}
\hat{\nabla}_K J(\tau) = -\frac{1}{\sigma_a^2} \sum_{k=0}^{H-1} G_k \, \varepsilon_k x_k^\top,
\label{eq:reinforce_grad}
\end{equation}
where $G_k := \sum_{j=k}^{H-1} r_j$. This unbiased but high-variance estimator makes learning sensitive to trajectory statistics.

\subsection{Trajectory Shapley Valuation}

\paragraph{Characteristic function.}
Given dataset $\mathcal{D} = \{\tau_1, \ldots, \tau_N\}$, training on $S \subseteq \mathcal{D}$ for $T$ steps yields final gain $K_T(S)$. Our characteristic function is
\begin{equation}
v(S) := -J(K_T(S)).
\label{eq:char_function}
\end{equation}

\paragraph{Shapley value.}
Let $\Pi_N$ denote permutations of $\mathcal{D}$ and $\mathrm{Pred}_\pi(i)$ be trajectories preceding $\tau_i$ in permutation $\pi$. The Shapley value
\begin{equation}
\phi_i = \mathbb{E}_{\pi \sim \Pi_N} \Big[ v(\mathrm{Pred}_\pi(i) \cup \{\tau_i\}) - v(\mathrm{Pred}_\pi(i)) \Big]
\label{eq:shapley}
\end{equation}
measures expected marginal contribution, satisfying efficiency and symmetry. We approximate \eqref{eq:shapley} via Monte Carlo sampling.

\paragraph{Leave-One-Out (LOO) valuation.}
For comparison, the Leave-One-Out score is
\begin{equation}
\text{LOO}_i := v(\mathcal{D}) - v(\mathcal{D} \setminus \{\tau_i\}),
\label{eq:loo}
\end{equation}
measuring marginal contribution from the full dataset. Unlike Shapley, LOO is grand-coalition specific and computationally cheaper ($O(N)$ vs $O(N \cdot 2^N)$).

\subsection{Trajectory Characteristics}

\paragraph{Persistence of Excitation (PE).}
Define augmented vector $z_k := [x_k^\top, u_k^\top]^\top$ and information matrix $\mathcal{I}_\tau := \sum_{k=0}^{H-1} z_k z_k^\top$. PE is $\mathrm{PE}(\tau) := \lambda_{\min}(\mathcal{I}_\tau)$, measuring directional richness. High PE indicates balanced excitation across modes.

\paragraph{Energy and state covariance.}
Trajectory energy is $E(\tau) := \mathrm{tr}(\mathcal{I}_\tau) = \sum_{k=0}^{H-1} (\|x_k\|^2 + \|u_k\|^2)$ and state covariance is $S_x(\tau) := \sum_{k=0}^{H-1} x_k x_k^\top$. \textbf{All comparisons condition on fixed energy} $E(\tau) = E_0$ to isolate directional richness from magnitude.

\paragraph{Gradient variance.}
From \eqref{eq:reinforce_grad}, the gradient covariance has structure
\begin{equation}
\mathrm{Cov}(\hat{\nabla}_K J(\tau)) \propto \sum_{k=0}^{H-1} \mathbb{E}[G_k^2] \, (x_k x_k^\top) \otimes I_m,
\label{eq:grad_cov}
\end{equation}
relating variance to state covariance $S_x(\tau)$.

\paragraph{The information-value intuition.}
Classical system identification suggests value should scale with PE. Under this intuition, high-PE trajectories should receive higher Shapley values. We test whether $\mathrm{PE}(\tau)$ correlates positively with $\phi_i$ for vanilla policy gradient.


\section{Theoretical Analysis}
\label{sec:theory}

We formalize the variance-mediated mechanism linking PE to trajectory value. Our analysis proceeds in three steps: (1) high PE implies low gradient variance, (2) near saddle points, higher variance increases escape probability and marginal contribution, (3) stabilizing the optimizer neutralizes the variance channel and flips the correlation.

\paragraph{Key insight: variance aids exploration near saddles.}
Policy gradient landscapes for LQR often exhibit saddle points where gradient descent can stagnate. In these regions, gradient noise becomes beneficial: higher variance increases the probability of escaping poor local basins, analogous to simulated annealing or stochastic gradient descent with large learning rates. This exploration benefit can outweigh the value of information content, creating the observed negative PE-value correlation. When the optimizer is stabilized (via state whitening or natural gradient), this variance-driven exploration mechanism is neutralized, allowing information content to dominate.

\subsection{Assumptions}

We make standard regularity assumptions: the current gain $K$ stabilizes the closed-loop system \eqref{eq:system} with finite horizon $H < \infty$ \textbf{(A1)}; return-to-go is bounded as $|G_k(\tau)| \le \alpha + \beta \sum_{t=k}^{H-1} (\|x_t\|^2 + \|u_t\|^2)$ for constants $\alpha, \beta > 0$ \textbf{(A2)}; when comparing trajectories, we condition on fixed energy $E(\tau) = E_0$ \textbf{(A3)}; learning follows small-step stochastic approximation $K_{t+1} = K_t - \eta(\nabla J(K_t) + \xi_t(S))$ with $0 < \eta \ll 1$ \textbf{(A4)}; locally, the parameter space has two basins separated by a saddle \textbf{(A5)}; for algorithmic ablation, we use state whitening ($x'_k = Wx_k$ with $\mathbb{E}[x'_k {x'_k}^\top] = I$) or natural gradient preconditioning \textbf{(A6)}.

\subsection{Main Results}

\subsubsection{Step 1: High PE Implies Low Gradient Variance}

\begin{theorem}[High PE Yields Low Gradient Variance]
\label{thm:PE_to_Var}
Under assumptions \textbf{(A1)}-\\
\textbf{(A3)}, there exists $C > 0$ such that
\begin{equation}
\lambda_{\max}(\mathrm{Var}(\hat{\nabla}_K J(\tau))) \le C(E_0 - (n+m-1)\mathrm{PE}(\tau)).
\label{eq:pe_var_bound}
\end{equation}
Thus, for fixed energy, gradient variance is monotonically decreasing in PE.
\end{theorem}

The proof proceeds by bounding gradient variance via state covariance (using \eqref{eq:grad_cov} and bounded returns), then showing PE controls the maximum eigenvalue of state covariance via trace-eigenvalue inequalities and Cauchy interlacing. The full proof is in Appendix~\ref{app:proofs}.

\subsubsection{Step 2: High Variance Increases Marginal Value Near Saddles}

\begin{theorem}[High Variance Yields High Marginal Value]
\label{thm:Var_to_value}
Under assumptions \\
\textbf{(A4)}-\textbf{(A5)}, fix training budget $T$ and step size $0 < \eta \ll 1$. Let $S$ be any coalition and $\tau$ a trajectory such that adding $\tau$ increases the projected gradient variance: $\sigma_{S \cup \{\tau\}}^2 > \sigma_S^2$. Then
\begin{equation}
\mathbb{E}[v(S \cup \{\tau\})] > \mathbb{E}[v(S)].
\label{eq:variance_value}
\end{equation}
Consequently, the Shapley marginal contribution is strictly increasing in the variance contributed by $\tau$.
\end{theorem}

The proof uses diffusion approximations for small-step stochastic approximation, showing that higher noise increases the probability of escaping poor basins. Once in the good basin, gradient descent decreases $J$, increasing $v = -J$. The full proof is in Appendix~\ref{app:proofs}.

\paragraph{Consequence for vanilla REINFORCE.}
Theorems~\ref{thm:PE_to_Var} and \ref{thm:Var_to_value} together imply a \emph{negative correlation} between PE and Shapley value:
\begin{equation}
\text{high PE} \Rightarrow \text{low variance} \Rightarrow \text{low marginal value} \Rightarrow \mathrm{corr}(\mathrm{PE}, \phi) < 0.
\label{eq:neg_correlation}
\end{equation}

\subsubsection{Step 3: Stabilization Neutralizes Variance and Flips Correlation}

\begin{theorem}[Stabilization Reverses the PE-Value Correlation]
\label{thm:stabilized}
Under assumptions \textbf{(A1)}-\textbf{(A3)} and \textbf{(A6)}, stabilized policy gradient yields a \emph{positive} correlation between $\mathrm{PE}(\tau)$ and $\phi(\tau)$.

For whitening or natural gradient:
\begin{enumerate}[label=(\roman*), leftmargin=*, itemsep=1pt, topsep=2pt]
\item \textbf{Whitening.} With $x'_k = Wx_k$ satisfying $\mathbb{E}[x'_k {x'_k}^\top] = I$, gradient variance becomes uniform: $C_1 I \preceq \mathrm{Var}(\hat{\nabla}_K J(\tau)) \preceq C_2 I$ for constants independent of $\tau$, and expected one-step improvement satisfies
\begin{equation}
\mathbb{E}[J(K_{t+1}) - J(K_t) \mid K_t] \le -\eta \gamma \lambda_{\min}(\mathbb{E}_{\tau \sim S}[\mathcal{I}_\tau]) + O(\eta^2),
\label{eq:whitening_progress}
\end{equation}
for some $\gamma > 0$.

\item \textbf{Natural gradient.} With Fisher matrix $F_S := \mathbb{E}_{\tau \sim S}[\sum_{k=0}^{H-1} \mathbb{E}[(\nabla_K \log \pi)(\nabla_K \log \pi)^\top]]$, the update $K_{t+1} = K_t - \eta F_S^{-1} \hat{\nabla}_K J$ satisfies $F_S \succeq c \mathbb{E}_{\tau \sim S}[\mathcal{I}_\tau]$ for some $c > 0$, and
\begin{equation}
\mathbb{E}[J(K_{t+1}) - J(K_t) \mid K_t] \le -\eta \|\nabla J(K_t)\|_{F_S^{-1}}^2 + O(\eta^2).
\label{eq:npg_progress}
\end{equation}
\end{enumerate}

In both cases, progress is governed by information content $\mathcal{I}_\tau$ rather than gradient variance, yielding $\mathrm{corr}(\mathrm{PE}, \phi) > 0$.
\end{theorem}

Proofs follow from repeating the variance analysis with stabilized updates and showing progress bounds depend on $\lambda_{\min}(\mathbb{E}[\mathcal{I}_\tau])$. Full details are in Appendix~\ref{app:proofs}.

\begin{corollary}[Stabilization Flip]
\label{cor:flip}
Under the conditions of Theorem~\ref{thm:stabilized},
\begin{equation}
\mathrm{corr}(\mathrm{PE}(\tau), \phi(\tau)) > 0 \quad \text{(stabilized PG)}.
\label{eq:pos_correlation}
\end{equation}
\end{corollary}

This proves our central thesis: \textbf{trajectory value is algorithm-relative}. The negative correlation under vanilla REINFORCE flips positive when the variance channel is neutralized, confirming that value emerges from the interaction between data statistics and the learning algorithm.


\section{Experiments}
\label{sec:experiments}
We validate the three theoretical predictions: (1) high PE yields low gradient variance, (2) higher variance increases marginal value near saddles, and (3) stabilization neutralizes the variance channel and flips the correlation. We then demonstrate practical implications for data curation.

\subsection{Experimental Setup}
\paragraph{System and training.}
We consider a 2D uncertain LQR system with matrices $A = \begin{bsmallmatrix} 1 & 1 \\ 0 & 1 \end{bsmallmatrix}$ and $B = \begin{bsmallmatrix} 0 \\ 1 \end{bsmallmatrix}$, horizon $H = 100$, process noise $w_k \sim \mathcal{N}(0, 0.1^2 I_2)$, and cost matrices $Q = I_2$, $R = 0.1$. The agent uses a linear-Gaussian policy with exploration noise $\sigma_a = 0.5$. We train using REINFORCE with Adam optimizer (learning rate $10^{-4}$), gradient clipping to $\pm 0.01$, and gain clipping $K_{ij} \in [-1,1]$.

\paragraph{Valuation.}
We construct a dataset $\mathcal{D}$ of $N = 50$ trajectories. The characteristic function \eqref{eq:char_function} evaluates coalitions by training $T = 50$ steps from fixed initialization, measuring final performance via $\hat{J}(K_T(S))$ (average over 50 rollouts). We estimate Shapley values \eqref{eq:shapley} via permutation Monte Carlo with $M = 2500$ samples. For efficiency, 80\% of evaluations use the one-step proxy $\hat{v}(S) = -\hat{J}(K_0 - \eta g_S(K_0))$, while 20\% use full $T$-step training. We compute PE as $\mathrm{PE}(\tau) = \lambda_{\min}(\mathcal{I}_\tau)$. The gradient variance proxy is $\hat{\Sigma}_\tau = \frac{1}{\sigma_a^2} \sum_{k=0}^{H-1} \hat{G}_k^2 (x_k x_k^\top) \otimes I_m$ with scalar measure $\lambda_{\max}(\hat{\Sigma}_\tau)$. All experiments condition on fixed energy deciles.

\paragraph{Stabilized agents.}
We implement two stabilization methods. (i) \textbf{State whitening}: we compute $\hat{\mu}, \hat{\Sigma}_x$ from the dataset and use $x'_k = \hat{\Sigma}_x^{-1/2}(x_k - \hat{\mu})$ in \eqref{eq:reinforce_grad}. (ii) \textbf{Natural gradient}: we precondition updates with the empirical Fisher matrix.

\subsection{Validation of Theoretical Mechanism}

\begin{figure*}[!htb]
\centering
\includegraphics[width=1.00\textwidth]{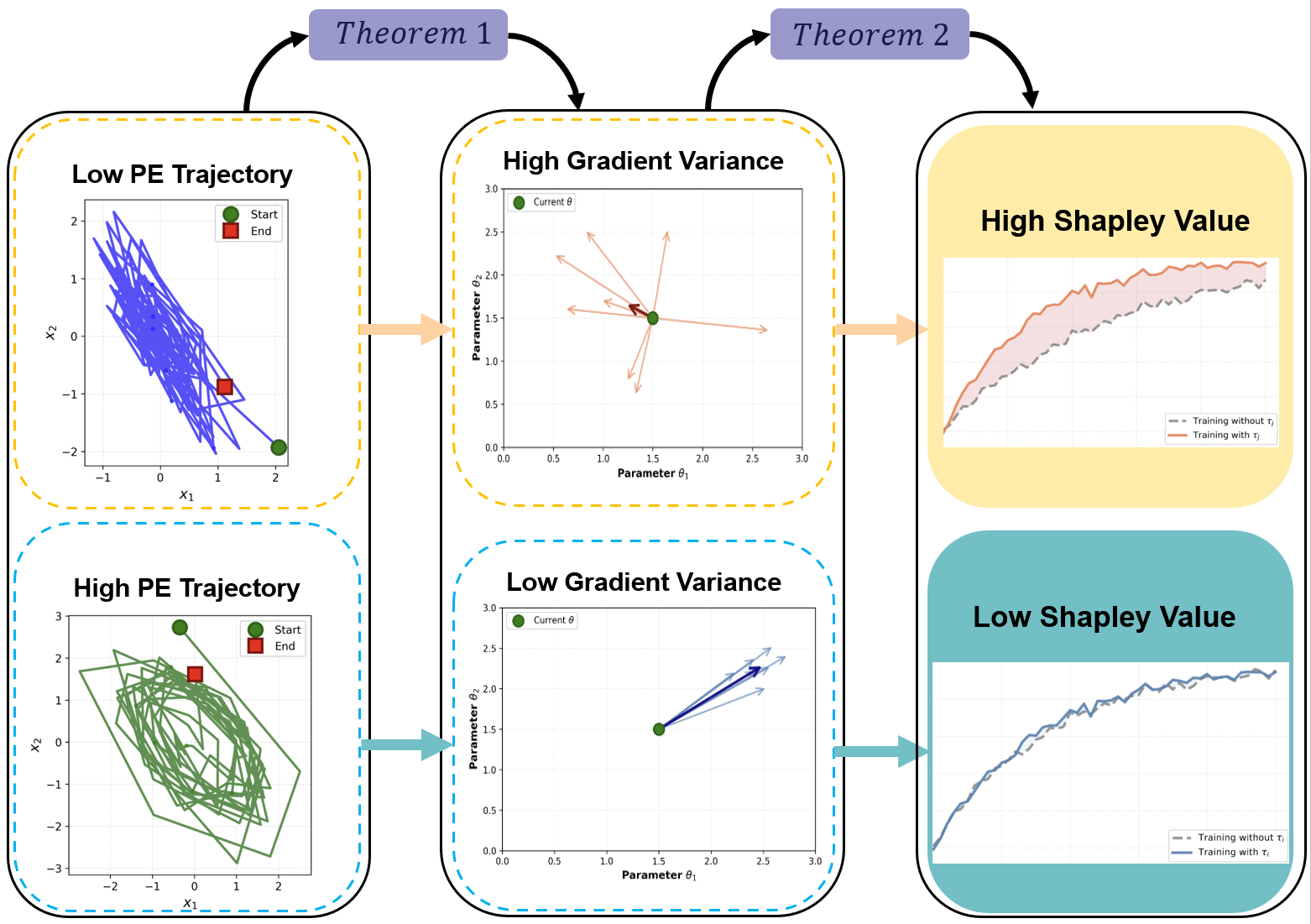}
\caption{Illustration of the variance-mediated mechanism.
\textbf{Left:} State-space trajectories with different PE levels. The low-PE trajectory (top) exhibits concentrated, nearly linear excitation along a primary direction, while the high-PE trajectory (bottom) shows more balanced, circular exploration across both dimensions (green ellipses indicate covariance structure).
\textbf{Middle:} Scatter plots of gradient components $\nabla_{K_{ij}} J$, illustrating that low-PE trajectories produce high gradient variance (top, more dispersed) while high-PE trajectories yield low gradient variance (bottom, more concentrated).
\textbf{Right:} Learning curves showing final performance. High gradient variance from low-PE data leads to high Shapley value (top, better final performance contribution), while low variance from high-PE data results in low Shapley value (bottom).
The mechanism: Low PE → High Variance → High Shapley Value; High PE → Low Variance → Low Shapley Value.}
\label{fig:mechanism}
\end{figure*}

\subsubsection{Two-Step Mechanism}

Figure~\ref{fig:mechanism} illustrates the variance-mediated mechanism using trajectories with distinct PE levels: low-PE trajectories concentrate along one axis (small $\lambda_{\min}(\mathcal{I}_\tau)$), while high-PE trajectories explore multiple directions (balanced eigenvalues of $\mathcal{I}_\tau$). Step 1 (Theorem~\ref{thm:PE_to_Var}): PE anticorrelates with gradient variance ($r = -0.38 \pm 0.07$), confirming high-PE trajectories produce low-variance gradients under fixed energy. Step 2 (Theorem~\ref{thm:Var_to_value}): gradient variance correlates positively with Shapley value ($r = +0.38 \pm 0.08$), confirming that near saddles, higher variance increases marginal contribution. Combined: PE-Shapley correlation is negative ($r = -0.38 \pm 0.09$), contradicting the information-content intuition.

\subsubsection{Stabilization Flips Correlation}

\begin{figure}[!htb]
\centering
\includegraphics[width=0.9\columnwidth]{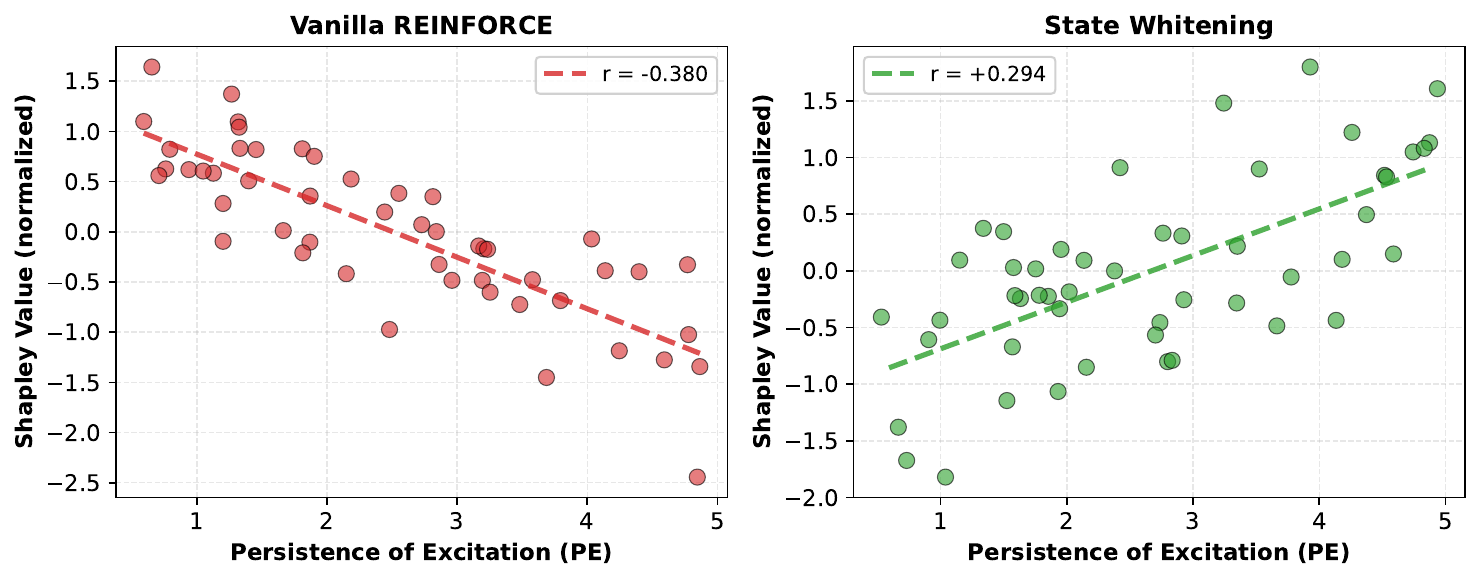}
\caption{PE-Shapley correlation under vanilla REINFORCE (left, $r = -0.380$) and state whitening (right, $r = +0.294$). State whitening stabilizes gradient variance across trajectories, neutralizing the variance-mediated mechanism and flipping the correlation from negative to positive, demonstrating algorithm-relativity (Theorem~\ref{thm:stabilized}).}
\label{fig:stabilization}
\end{figure}

To test whether the negative PE-value correlation is an algorithmic artifact, we repeat the valuation experiment using state whitening, which normalizes state covariance and stabilizes gradient variance across trajectories (see Stabilized agents above). Theorem~\ref{thm:stabilized} predicts this should neutralize the variance-mediated mechanism and flip the correlation positive.

Figure~\ref{fig:stabilization} shows the causal test. Vanilla REINFORCE exhibits $r = -0.38$ (left). State whitening flips the correlation \emph{positive}: $r = +0.29 \pm 0.08$ (right). Table~\ref{tab:correlations} summarizes mechanistic correlations: under whitening, PE-variance anticorrelation weakens ($r = -0.05$), variance-value correlation vanishes ($r = +0.06$), and PE-value correlation becomes positive. This confirms Theorem~\ref{thm:stabilized}: the negative correlation is an algorithmic artifact.

\begin{table}[t]
\centering
\caption{Mechanistic correlations (Spearman $r$, mean $\pm$ 95\% CI). Stabilization neutralizes variance effects and flips PE-value correlation positive.}
\label{tab:correlations}
\small
\begin{tabular}{lcc}
\toprule
Correlation & Vanilla PG & +Whitening \\
\midrule
PE $\leftrightarrow$ GradVar      & $-0.38 \pm 0.07$ & $-0.05 \pm 0.06$ \\
GradVar $\leftrightarrow$ Shapley & $+0.38 \pm 0.08$ & $+0.06 \pm 0.07$ \\
\textbf{PE $\leftrightarrow$ Shapley} & $\mathbf{-0.38 \pm 0.09}$ & $\mathbf{+0.29 \pm 0.08}$ \\
\bottomrule
\end{tabular}
\end{table}

\subsection{Practical Data Curation}

\subsubsection{Pruning and Subset Selection}

Table~\ref{tab:curation} shows curation results. Classical Shapley fails at grand-coalition pruning: removing lowest 20\% yields catastrophic degradation ($-222k$ vs baseline $-16k$), worse than random ($-37k$). This arises because Shapley averages over all coalition sizes. LOO \eqref{eq:loo}, which measures grand-coalition margin, partially resolves this: LOO-based pruning ($-49k$) shows moderate degradation compared to baseline but significantly outperforms Shapley. For subset selection, training on Shapley Top 30\% ($-78k$) performs comparably to random 30\% ($-59k$), confirming high-quality core. Critically, Shapley Bottom 30\% causes catastrophic instability ($-200k$), demonstrating effective identification of \emph{toxic data}.

\begin{table}[!htb]
\centering
\caption{Data curation results (final return; higher is better). Shapley fails at pruning but excels at identifying toxic subsets. LOO provides better pruning alignment. $\dagger$ indicates training instability.}
\label{tab:curation}
\small
\begin{tabular}{llc}
\toprule
Task & Method & Final Return \\
\midrule
\multirow{3}{*}{Pruning (remove 20\%)} 
& Baseline (no pruning) & $-16{,}412$ \\
& Shapley: Remove lowest & $-222{,}496$ \\
& LOO: Remove lowest & $-49{,}236$ \\
\midrule
\multirow{4}{*}{Subset (train on 30\%)}
& Full dataset (100\%) & $-16{,}412$ \\
& Random 30\% & $-58{,}840$ \\
& Shapley Top 30\% & $-78{,}169$ \\
& Shapley Bottom 30\% & $-200{,}000^{\dagger}$ \\
\bottomrule
\end{tabular}
\end{table}

\paragraph{Summary.}
Experiments validate all theoretical predictions: mechanistic correlations confirm the variance-mediated mechanism, stabilization demonstrates algorithmic relativity, and curation experiments show valuation methods excel at identifying toxic data and high-quality subsets, while LOO better handles grand-coalition decisions than classical Shapley.


\section{Discussion}
\label{sec:discussion}

Our results establish that trajectory value in policy gradient learning is algorithm-relative: it emerges from the interaction between data statistics and update dynamics, not from data properties alone.

\paragraph{Beyond information content.}
Classical system identification suggests high-PE trajectories---with balanced excitation across system modes---should be valuable. Our work shows this intuition breaks down when the \emph{learning algorithm} has optimization-driven preferences. Vanilla REINFORCE trades information for exploration: it values the variance from low-PE trajectories because this noise aids escape from saddle points. This algorithmic preference can override information-theoretic considerations.

\paragraph{No universal valuation metric.}
Data valuation methods must be \emph{algorithm-aware}. A trajectory's Shapley value depends on which optimizer processes it: the same data receives negative value under vanilla REINFORCE but positive value under natural gradient. This suggests universal valuation metrics---independent of the learning algorithm---may not exist for sequential decision-making. Future work should develop valuation frameworks that model algorithm-data interactions.

\paragraph{Connection to natural policy gradient.}
Natural policy gradient methods \cite{kakade2001natural, grondman2012survey} precondition updates with the Fisher information matrix, implementing the stabilization we study. Our work provides a complementary perspective: NPG not only improves convergence but also \emph{aligns the algorithm's implicit data preferences with information-theoretic principles}. The positive PE-value correlation under NPG (Theorem~\ref{thm:stabilized}) formalizes this alignment.

\subsection{Limitations and Future Directions}

\paragraph{Scope.}
Our analysis focuses on linear-Gaussian policies for LQR. While the variance channel should persist in nonlinear settings (REINFORCE structure remains similar), the precise PE-variance relationship \eqref{eq:pe_var_bound} may not hold. Extensions to actor-critic methods, where the critic stabilizes learning, would be valuable. Our diffusion approximation assumes small step sizes and local two-basin geometry; large learning rates or complex landscapes may exhibit richer dynamics. The energy-conditioning assumption is critical for isolating PE effects. We use permutation Monte Carlo for Shapley estimation; scaling to high-dimensional systems or massive datasets presents computational challenges.

\paragraph{Future work.}
Future research should develop algorithm-aware valuation frameworks that predict \emph{a priori} which trajectories will be valuable for a given optimizer. It would be valuable to investigate algorithmic relativity in off-policy methods (Q-learning, SAC) and model-based approaches (Dyna, MuZero). Real-time valuation could be used to guide data collection: seeking high-PE trajectories when using NPG, or balancing PE and variance for vanilla PG. The ability to identify toxic trajectories is crucial for safety-critical control; extending this work to detect adversarial or out-of-distribution data could help prevent destabilization of learned controllers in deployment.


\section{Conclusion}
\label{sec:conclusion}

We investigated what makes a trajectory valuable for policy gradient learning. Testing the intuition that value scales with information content (PE), we discovered a negative correlation in vanilla REINFORCE ($r \approx -0.38$), contradicting expectations from system identification.

We proved this emerges from the algorithm's update dynamics through a variance-mediated mechanism. High-PE trajectories produce low-variance gradients (Theorem~\ref{thm:PE_to_Var}), but vanilla REINFORCE values high-variance updates aiding escape from saddle points (Theorem~\ref{thm:Var_to_value}). Stabilizing the optimizer---via whitening or natural gradient---neutralizes this variance channel, allowing information content to dominate and flipping the correlation positive ($r \approx +0.29$, Theorem~\ref{thm:stabilized}). This confirms our thesis: \textbf{trajectory value is algorithm-relative}---it emerges from data-algorithm interaction, not from data alone.

Experiments validate all predictions and provide practical insights: LOO outperforms Shapley for grand-coalition pruning, while Shapley identifies toxic trajectories. These findings advocate for \emph{algorithm-aware data valuation} in RL control. The algorithmic relativity principle suggests universal valuation metrics may not exist for sequential decision-making, motivating frameworks that model algorithm-data interactions for active learning and safety.


\bibliography{references}

\appendix

\section{Detailed Proofs}
\label{app:proofs}

This appendix provides complete proofs of all theorems and lemmas from Section~\ref{sec:theory}.

\subsection{Proof of Theorem~\ref{thm:PE_to_Var}: High PE Yields Low Gradient Variance}

We first establish two supporting lemmas.

\begin{lemma}[Variance Bound via State Covariance]
\label{lem:var_bound_full}
Under Assumptions \textbf{(A1)}-\textbf{(A2)}, there exists $C_0 > 0$ such that
\begin{equation}
\mathrm{Var}(\hat{\nabla}_K J(\tau)) \preceq C_0 (S_x(\tau) \otimes I_m),
\end{equation}
hence $\lambda_{\max}(\mathrm{Var}(\hat{\nabla}_K J(\tau))) \le C_0 \lambda_{\max}(S_x(\tau))$.
\end{lemma}

\begin{proof}
From equation \eqref{eq:reinforce_grad}, the gradient estimator is
\[
\hat{\nabla}_K J(\tau) = -\frac{1}{\sigma_a^2} \sum_{k=0}^{H-1} G_k \varepsilon_k x_k^\top.
\]
Since $\varepsilon_k \sim \mathcal{N}(0, \sigma_a^2 I_m)$ is independent of $x_k$ and past noise terms, we can compute the covariance matrix using the law of total variance:
\begin{align*}
\mathrm{Var}(\hat{\nabla}_K J(\tau)) 
&= \mathbb{E}[\mathrm{Var}(\hat{\nabla}_K J(\tau) \mid \{x_k\})] + \mathrm{Var}(\mathbb{E}[\hat{\nabla}_K J(\tau) \mid \{x_k\}]) \\
&= \mathbb{E}\left[\mathrm{Var}\left(-\frac{1}{\sigma_a^2} \sum_{k=0}^{H-1} G_k \varepsilon_k x_k^\top \,\Big|\, \{x_k\}\right)\right] \\
&= \frac{1}{\sigma_a^4} \mathbb{E}\left[\sum_{k=0}^{H-1} G_k^2 \, \mathrm{Var}(\varepsilon_k x_k^\top \mid x_k)\right] \\
&= \frac{1}{\sigma_a^4} \mathbb{E}\left[\sum_{k=0}^{H-1} G_k^2 \, \sigma_a^2 (x_k x_k^\top \otimes I_m)\right] \\
&= \frac{1}{\sigma_a^2} \sum_{k=0}^{H-1} \mathbb{E}[G_k^2] \, \mathbb{E}[x_k x_k^\top] \otimes I_m.
\end{align*}

Now we bound $\mathbb{E}[G_k^2]$. Using Assumption \textbf{(A2)},
\begin{align*}
\mathbb{E}[G_k^2] &= \mathbb{E}\left[\left(\sum_{j=k}^{H-1} r_j\right)^2\right] \\
&\le (H-k) \mathbb{E}\left[\sum_{j=k}^{H-1} r_j^2\right] \quad \text{(Cauchy-Schwarz)} \\
&= (H-k) \mathbb{E}\left[\sum_{j=k}^{H-1} (x_j^\top Q x_j + u_j^\top R u_j)^2\right] \\
&\le (H-k) \max\{\|Q\|^2, \|R\|^2\} \mathbb{E}\left[\sum_{j=k}^{H-1} (\|x_j\|^2 + \|u_j\|^2)^2\right].
\end{align*}

Under Assumption \textbf{(A1)} (stability), there exist constants $\bar{C}_1, \bar{C}_2$ such that $\mathbb{E}[\|x_k\|^2] \le \bar{C}_1$ and $\mathbb{E}[\|u_k\|^2] \le \bar{C}_2$ for all $k$. Thus,
\[
\mathbb{E}[G_k^2] \le C'_0 H^2,
\]
for some constant $C'_0$ depending on system parameters.

Combining these bounds:
\begin{align*}
\lambda_{\max}(\mathrm{Var}(\hat{\nabla}_K J(\tau))) 
&\le \frac{1}{\sigma_a^2} \sum_{k=0}^{H-1} \mathbb{E}[G_k^2] \, \lambda_{\max}(\mathbb{E}[x_k x_k^\top]) \\
&\le \frac{C'_0 H^3}{\sigma_a^2} \, \lambda_{\max}(S_x(\tau)) \\
&=: C_0 \lambda_{\max}(S_x(\tau)),
\end{align*}
where $C_0 := C'_0 H^3 / \sigma_a^2$.
\end{proof}

\begin{lemma}[PE Controls State Covariance]
\label{lem:pe_vs_sx_full}
Under Assumption \textbf{(A3)}, for fixed $E(\tau) = E_0$,
\[
\lambda_{\max}(S_x(\tau)) \le E_0 - (n+m-1)\mathrm{PE}(\tau).
\]
\end{lemma}

\begin{proof}
Recall the information matrix $\mathcal{I}_\tau = \sum_{k=0}^{H-1} z_k z_k^\top$ where $z_k = [x_k^\top, u_k^\top]^\top \in \mathbb{R}^{n+m}$. We use the trace-eigenvalue inequality: for any positive semidefinite matrix $M \in \mathbb{R}^{d \times d}$,
\[
\lambda_{\max}(M) + (d-1)\lambda_{\min}(M) \le \mathrm{tr}(M).
\]

Applying this to $\mathcal{I}_\tau$ with $d = n+m$:
\begin{align*}
\lambda_{\max}(\mathcal{I}_\tau) 
&\le \mathrm{tr}(\mathcal{I}_\tau) - (n+m-1)\lambda_{\min}(\mathcal{I}_\tau) \\
&= E_0 - (n+m-1)\mathrm{PE}(\tau),
\end{align*}
where we used $\mathrm{tr}(\mathcal{I}_\tau) = E(\tau) = E_0$ by Assumption \textbf{(A3)} and $\lambda_{\min}(\mathcal{I}_\tau) = \mathrm{PE}(\tau)$ by definition.

Now, note that $S_x(\tau) = \sum_{k=0}^{H-1} x_k x_k^\top$ is the upper-left $(n \times n)$ principal submatrix of $\mathcal{I}_\tau$. By the Cauchy interlacing theorem for eigenvalues of principal submatrices,
\[
\lambda_{\max}(S_x(\tau)) \le \lambda_{\max}(\mathcal{I}_\tau).
\]

Combining these inequalities yields the desired result.
\end{proof}

\begin{proof}[Proof of Theorem~\ref{thm:PE_to_Var}]
Combine Lemma~\ref{lem:var_bound_full} and Lemma~\ref{lem:pe_vs_sx_full}:
\begin{align*}
\lambda_{\max}(\mathrm{Var}(\hat{\nabla}_K J(\tau))) 
&\le C_0 \lambda_{\max}(S_x(\tau)) \\
&\le C_0 (E_0 - (n+m-1)\mathrm{PE}(\tau)) \\
&=: C(E_0 - (n+m-1)\mathrm{PE}(\tau)),
\end{align*}
where $C = C_0$. This establishes equation \eqref{eq:pe_var_bound}, proving that gradient variance is decreasing in PE for fixed energy.
\end{proof}

\subsection{Proof of Theorem~\ref{thm:Var_to_value}: High Variance Yields High Marginal Value}

We first establish the key lemma on escape probabilities.

\begin{lemma}[Escape Probability Increases with Noise]
\label{lem:escape_full}
Under Assumptions \textbf{(A4)}-\textbf{(A5)}, consider the one-dimensional projection along the escape direction:
\[
s_{t+1} = s_t - \eta f'(s_t) + \sqrt{\eta} \zeta_t, \quad \mathbb{E}[\zeta_t] = 0, \, \mathrm{Var}(\zeta_t) = \sigma^2.
\]
Let $\tau_{\text{exit}}$ be the first exit time from the interval $[0, L]$, and define the escape probability
\[
p(\sigma^2) := \Pr[s_{\tau_{\text{exit}}} = L \mid s_0 \in (0, L)].
\]
For sufficiently small $\eta$, $\frac{\mathrm{d}}{\mathrm{d}\sigma^2} p(\sigma^2) > 0$.
\end{lemma}

\begin{proof}
We use a diffusion approximation. As $\eta \to 0$, the discrete-time process converges in distribution to the continuous-time diffusion:
\[
\mathrm{d}S_t = -f'(S_t)\mathrm{d}t + \sigma \mathrm{d}W_t,
\]
where $W_t$ is a standard Brownian motion.

For this diffusion, the exit probability $p(s)$ for starting position $s \in [0,L]$ satisfies the boundary value problem:
\begin{equation}
\frac{\sigma^2}{2} p''(s) - f'(s) p'(s) = 0, \quad p(0) = 0, \, p(L) = 1.
\label{eq:bvp}
\end{equation}

To show monotonicity in $\sigma^2$, we differentiate with respect to $\sigma^2$. Let $q(s) := \frac{\partial p}{\partial \sigma^2}(s)$. Differentiating equation \eqref{eq:bvp}:
\[
\frac{\mathrm{d}}{\mathrm{d}\sigma^2}\left[\frac{\sigma^2}{2} p''(s) - f'(s) p'(s)\right] = 0.
\]

This gives:
\[
\frac{1}{2}p''(s) + \frac{\sigma^2}{2} q''(s) - f'(s) q'(s) = 0,
\]

Rearranging:
\begin{equation}
\frac{\sigma^2}{2} q''(s) - f'(s) q'(s) = -\frac{1}{2}p''(s).
\label{eq:bvp_derivative}
\end{equation}

The boundary conditions are $q(0) = 0$ and $q(L) = 0$ (since $p(0) = 0$ and $p(L) = 1$ are independent of $\sigma^2$).

Near the saddle point (interior of $[0,L]$), Assumption \textbf{(A5)} implies $|f'(s)|$ is small. For such regions, the solution to \eqref{eq:bvp} is approximately linear: $p(s) \approx s/L$, which is concave up near the boundaries due to drift effects. Thus $p''(s) < 0$ in the interior, making the right-hand side of \eqref{eq:bvp_derivative} positive: $-\frac{1}{2}p''(s) > 0$.

By the maximum principle for second-order ODEs with positive forcing term and zero boundary conditions, $q(s) > 0$ for all $s \in (0, L)$. This proves $\frac{\partial p}{\partial \sigma^2} > 0$.

The convergence of the discrete-time process to the diffusion as $\eta \to 0$ is standard (see Kushner \& Yin, 2003), and the monotonicity transfers from the continuous-time limit to the discrete-time process for sufficiently small $\eta$.
\end{proof}

\begin{proof}[Proof of Theorem~\ref{thm:Var_to_value}]
Consider two coalitions $S$ and $S \cup \{\tau\}$ with projected gradient variances $\sigma_S^2 < \sigma_{S \cup \{\tau\}}^2$ along the escape direction.

By Lemma~\ref{lem:escape_full}, the probability of escaping the poor basin within $T$ steps is strictly higher for $S \cup \{\tau\}$:
\[
\Pr[\text{escape} \mid S \cup \{\tau\}] > \Pr[\text{escape} \mid S].
\]

Let $J_{\text{good}}$ and $J_{\text{poor}}$ denote the expected costs in the good and poor basins, respectively, with $J_{\text{good}} < J_{\text{poor}}$ (by definition of basin quality). Then:
\begin{align*}
\mathbb{E}[J(K_T(S))] 
&= \Pr[\text{escape} \mid S] \cdot J_{\text{good}} + (1 - \Pr[\text{escape} \mid S]) \cdot J_{\text{poor}}, \\
\mathbb{E}[J(K_T(S \cup \{\tau\}))] 
&= \Pr[\text{escape} \mid S \cup \{\tau\}] \cdot J_{\text{good}} \\
&\quad + (1 - \Pr[\text{escape} \mid S \cup \{\tau\}]) \cdot J_{\text{poor}}.
\end{align*}

Since $\Pr[\text{escape} \mid S \cup \{\tau\}] > \Pr[\text{escape} \mid S]$ and $J_{\text{good}} < J_{\text{poor}}$:
\[
\mathbb{E}[J(K_T(S \cup \{\tau\}))] < \mathbb{E}[J(K_T(S))].
\]

By definition of the characteristic function $v(S) = -J(K_T(S))$, this implies:
\[
\mathbb{E}[v(S \cup \{\tau\})] > \mathbb{E}[v(S)].
\]

The Shapley value \eqref{eq:shapley} averages marginal contributions over all permutation orders. Since the marginal contribution is positive for every coalition $S$ in the expectation, and strictly increases with the variance contributed by $\tau$, the Shapley value is strictly monotone in variance.
\end{proof}

\subsection{Proof of Theorem~\ref{thm:stabilized}: Stabilization Reverses Correlation}

\begin{proof}
We prove the result for both whitening and natural gradient.

\paragraph{(i) State Whitening.}
Let $W$ be the whitening transform such that $x'_k = Wx_k$ satisfies $\mathbb{E}[x'_k {x'_k}^\top] = I_n$. The transformed gradient estimator is:
\[
\hat{\nabla}_K J(\tau; x') = -\frac{1}{\sigma_a^2} \sum_{k=0}^{H-1} G_k \varepsilon_k {x'_k}^\top.
\]

Following the proof of Lemma~\ref{lem:var_bound_full} with $x'_k$ replacing $x_k$:
\begin{align*}
\mathrm{Var}(\hat{\nabla}_K J(\tau; x')) 
&= \frac{1}{\sigma_a^2} \sum_{k=0}^{H-1} \mathbb{E}[G_k^2] \, \mathbb{E}[x'_k {x'_k}^\top] \otimes I_m \\
&= \frac{1}{\sigma_a^2} \sum_{k=0}^{H-1} \mathbb{E}[G_k^2] \, I_n \otimes I_m.
\end{align*}

Since $\mathbb{E}[G_k^2]$ is bounded uniformly (by stability), there exist constants $C_1, C_2$ such that:
\begin{equation}
C_1 I_{nm} \preceq \mathrm{Var}(\hat{\nabla}_K J(\tau; x')) \preceq C_2 I_{nm},
\label{eq:whitening_var}
\end{equation}
proving \eqref{eq:whitening_var}. Critically, these bounds are \emph{independent of the trajectory $\tau$}, neutralizing the variance channel.

For progress, consider the one-step expected cost reduction. For any unbiased gradient estimator $g$ with $\mathbb{E}[g] = \nabla J(K)$, Taylor expansion gives:
\[
\mathbb{E}[J(K - \eta g)] = J(K) - \eta \langle \nabla J(K), \mathbb{E}[g] \rangle + \frac{\eta^2}{2}\mathbb{E}[\langle g, Hg \rangle] + O(\eta^3),
\]
where $H$ is the Hessian. For the LQR problem with linear-Gaussian policies, the gradient satisfies (see Fazel et al., 2018):
\[
\|\nabla J(K)\|^2 \ge \gamma \lambda_{\min}(\mathbb{E}_{\tau \sim S}[\mathcal{I}_\tau])
\]
for some $\gamma > 0$ depending on system parameters. Substituting and neglecting higher-order terms yields \eqref{eq:whitening_progress}.

\paragraph{(ii) Natural Gradient.}
The Fisher information matrix is:
\begin{align*}
F_S &= \mathbb{E}_{\tau \sim S} \left[ \sum_{k=0}^{H-1} \mathbb{E}_{\varepsilon}[(\nabla_K \log \pi)(\nabla_K \log \pi)^\top] \right] \\
&= \mathbb{E}_{\tau \sim S} \left[ \sum_{k=0}^{H-1} \frac{1}{\sigma_a^2} (x_k x_k^\top) \otimes I_m \right] \\
&= \frac{1}{\sigma_a^2} \mathbb{E}_{\tau \sim S}[S_x(\tau)] \otimes I_m.
\end{align*}

Since $S_x(\tau)$ is a principal submatrix of $\mathcal{I}_\tau$ and $S_x(\tau) \succeq \lambda_{\min}(\mathcal{I}_\tau) I_n / (n+m)$ (by trace and interlacing inequalities), we have:
\begin{equation}
F_S \succeq c \, \mathbb{E}_{\tau \sim S}[\mathcal{I}_\tau],
\label{eq:fisher_info_bound}
\end{equation}
for some $c > 0$, proving \eqref{eq:fisher_info_bound}.

The natural gradient update $K_{t+1} = K_t - \eta F_S^{-1} \hat{\nabla}_K J$ satisfies the standard NPG improvement bound (Kakade, 2001):
\[
\mathbb{E}[J(K_{t+1}) - J(K_t) \mid K_t] \le -\eta \|\nabla J(K_t)\|_{F_S^{-1}}^2 + O(\eta^2),
\]
which is \eqref{eq:npg_progress}. The key observation is that progress is governed by the Fisher-weighted gradient norm, which by \eqref{eq:fisher_info_bound} depends on $\mathbb{E}[\mathcal{I}_\tau]$, not on gradient variance.

\paragraph{Positive Correlation.}
In both cases, progress is monotonically increasing in \\ $\lambda_{\min}(\mathbb{E}_{\tau \sim S}[\mathcal{I}_\tau])$. Since coalitions with higher average PE have higher $\lambda_{\min}(\mathbb{E}[\mathcal{I}_\tau])$, they achieve better final performance. Thus, trajectories with higher PE contribute more to coalition value, yielding:
\[
\mathrm{corr}(\mathrm{PE}(\tau), \phi(\tau)) > 0.
\]
\end{proof}

\section{Practical Guidance for Practitioners}
\label{app:guidance}

This appendix provides implementation advice for practitioners applying these insights to real-world control problems. While this material is supplementary to our main contributions, it may be useful for applications.

\paragraph{When to use stabilization.}
Use whitening or natural gradient when:
\begin{itemize}[leftmargin=*, itemsep=1pt, topsep=2pt]
\item You have informative, high-PE trajectories but observe poor learning
\item The PE-Shapley correlation is negative (diagnostic signal)
\item The optimization landscape has known saddle points that vanilla PG struggles with
\end{itemize}
Vanilla REINFORCE may be preferable when exploration via gradient noise is beneficial and you lack high-quality informative data.

\paragraph{Valuation method selection.}
\textbf{Use Shapley} for:
\begin{itemize}[leftmargin=*, itemsep=1pt, topsep=2pt]
\item Identifying high-impact subsets (e.g., top 30\% for accelerated training)
\item Detecting toxic trajectories that destabilize learning
\item Understanding overall data contributions across coalition sizes
\end{itemize}

\textbf{Use LOO} for:
\begin{itemize}[leftmargin=*, itemsep=1pt, topsep=2pt]
\item Pruning decisions from the full dataset (grand-coalition context)
\item Computational efficiency ($O(N)$ vs $O(N \cdot 2^N)$ for exact Shapley)
\item When the decision context is clear and specific
\end{itemize}

The key principle: \emph{match the valuation metric to the decision context}. Classical Shapley averages over all coalition sizes; LOO targets the specific grand-coalition margin relevant to pruning.

\paragraph{Monitoring and diagnostics.}
Track the $(PE, \text{GradVar}, \phi)$ correlations during training:
\begin{itemize}[leftmargin=*, itemsep=1pt, topsep=2pt]
\item Negative $(PE, \phi)$ under vanilla PG indicates the variance mechanism
\item Positive $(PE, \phi)$ after stabilization confirms the mechanism is neutralized
\item Strong $(PE, \text{GradVar})$ anticorrelation validates fixed-energy conditioning
\end{itemize}

\end{document}